\documentclass{article}
\usepackage{ijcai23}

\usepackage{times}
\usepackage{soul}
\usepackage{url}
\usepackage[hidelinks]{hyperref}
\usepackage[utf8]{inputenc}
\usepackage[small]{caption}
\usepackage[]{subcaption}
\usepackage{graphicx}
\usepackage{amsmath}
\usepackage{amsthm}
\usepackage{booktabs}
\usepackage{algorithm}
\usepackage{algorithmic}
\usepackage[switch]{lineno}

\usepackage{bm}
\usepackage{wrapfig}
\usepackage{amssymb}
\usepackage{graphicx}
\usepackage{newfloat}
\usepackage{enumitem}

\DeclareMathOperator*{\argmax}{argmax}

\newtheorem{definition}{Definition}
\newtheorem{claim}{Claim}

\title{Improving Heterogeneous Model Reuse by Density Estimation}

\author{
Anke Tang $^{1,2}$
\and Yong Luo $^{1,2\ast}$
\and Han Hu $^3$
\and Fengxiang He $^4$
\and \\ Kehua Su $^1$\footnote{Corresponding authors: Yong Luo, Kehua Su.}
\and Bo Du $^{1,2}$
\and Yixin Chen $^5$
\and Dacheng Tao $^6$
\affiliations
$^1$ School of Computer Science, National Engineering Research Center for Multimedia Software, Institute of Artificial Intelligence and Hubei Key Laboratory of Multimedia and Network Communication Engineering, Wuhan University, China.
$^2$ Hubei Luojia Laboratory, Wuhan, China. \\
$^3$ School of Information and Electronics, Beijing Institute of Technology, China. \\
$^4$ JD Explore Academy, JD.com, Inc., China. \\
$^5$ Department of CSE, Washington University in St. Louis, USA. \\
$^6$ The University of Sydney, Australia.
\emails
\{anketang, luoyong\}@whu.edu.cn,
hhu@bit.edu.cn,
fengxiang.f.he@gmail.com, \\
\{skh, dubo\}@whu.edu.cn,
chen@cse.wustl.edu,
dacheng.tao@gmail.com
}

\begin{document}

\maketitle

\begin{abstract}
This paper studies multiparty learning, aiming to learn a model using the private data of different participants.
Model reuse is a promising solution for multiparty learning, assuming that a local model has been trained for each party.
Considering the potential sample selection bias among different parties, some heterogeneous model reuse approaches have been developed.
However, although pre-trained local classifiers are utilized in these approaches, the characteristics of the local data are not well exploited. 
This motivates us to estimate the density of local data and design an auxiliary model together with the local classifiers for reuse.
To address the scenarios where some local models are not well pre-trained, we further design a multiparty cross-entropy loss for calibration.
Upon existing works, we address a challenging problem of heterogeneous model reuse from a decision theory perspective and take advantage of recent advances in density estimation.
Experimental results on both synthetic and benchmark data demonstrate the superiority of the proposed method.
\end{abstract}

\section{Introduction}

In recent years, leveraging centralized large-scale data by deep learning has achieved remarkable success in various application domains.
However, there are many scenarios where different participants separately collect data, and data sharing is prohibited due to the privacy legislation and high transmission cost.
For example, in some specific applications, such as medicine and autonomous driving, learnable data is inherently privacy-related and decentralized, and each local dataset is often insufficient to train a reliable prediction model~\cite{savage2017calculating,rajkomar2019machine}.
Therefore, multiparty learning is proposed to learn a reliable model using separated private datasets without sharing trainable samples~\cite{NIPS2010_0d0fd7c6}.

Most of the existing multiparty learning systems focus on training a shared global model to simultaneously achieve satisfactory accuracy and protect data privacy. These systems usually assume that each party trains a homogeneous local model, e.g., training neural networks with the same architecture~\cite{shokri2015privacy}. This makes it possible to directly average model parameters or aggregate gradient information~\cite{warnat2021swarm,mcmahan2016federated,li2019rsa}.
Some other works assume that each party has already trained a local model on its local dataset, and then apply model reuse to learn a global model~\cite{NIPS2010_0d0fd7c6,yang_deep_2017}. A typical example is the heterogeneous model reuse (HMR) method presented in~\cite{Wu2019}. Since only the output predictions of local models are utilized to derive a global model, the data can be non-i.i.d distributed and the architectures of different local models can vary among different parties. In addition, training of the global model can be quite efficient and data transmission cost can be significantly reduced.

There also exist some other model reuse approaches that may be utilized for multiparty learning. For example, pre-trained nonlinear auxiliary classifiers are adapted to new object functions in \cite{li2012efficient}.
Alternatively, the simple voting strategy can be adapted and improved to ensemble local models~\cite{zhou2012ensemble,Wu2019}.
In addition to the local models, a few works consider the design of specification to assist model selection and weight assignment~\cite{Ding2020,9447164}.
However, some important characteristics of the local data, such as the data density information are simply ignored in these approaches.

This motivates us to propose a novel heterogeneous model reuse method from a decision theory perspective that exploits the density information of local data.
In particular, in addition to the local model provided by each party, we estimate the probability density function of local data and design an auxiliary generative probabilistic model for reuse.
The proposed model ensemble strategy is based on the rules of Bayesian inference.
By feeding the target samples into the density estimators, we can obtain confidence scores of the accompanying local classifier when performing prediction for these samples.
Focusing on the semantic outputs, the heterogeneous local models are treated as black boxes and are allowed to abstain from making a final decision if the confidence is low for a certain sample in the prediction.
Therefore, aided by the density estimation, we can assign sample-level weight to the prediction of the local classifier.
Besides, when some local models are insufficiently trained on local datasets, we design a multiparty cross-entropy loss for calibration.
The designed loss automatically assigns a larger gradient to the local model that provides a more significant density estimation, and thus, enables it to obtain faster parameter updates.

To summarize, the main contributions of this paper are:
\begin{itemize}
\item we propose a novel model reuse approach for multiparty learning, where the data density is explored to help the reuse of biased models trained on local datasets to construct a reliable global model;
\item we design a multiparty cross-entropy loss, which can further optimize deep global model in an end-to-end manner.
\end{itemize}

We conduct experiments on both synthetic and benchmark data for image classification tasks.
The experimental results demonstrate that our method is superior to some competitive and recently proposed counterparts~\cite{Wu2019,9447164}.
Specifically, we achieve a significant $17.4\%$ improvement compared with \cite{9447164} in the case of three strict disjoint parties on the benchmark data.
Besides, the proposed calibration operation is proven to be effective even when local models are random initialized without training.


\section{Related Work}

In this section, we briefly summarize related works on multiparty learning and model reuse.

\subsection{Multiparty Learning}
Secure multiparty computation (SMC) \cite{yao1986generate,lindell2005secure} naturally involves multiple parties. The goal of SMC is to design a protocol, which is typically complicated, to exchange messages without revealing private data and compute a function on multiparty data. SMC requires communication between parties, leading to a huge amount of communication overhead. The complicated computation protocols are another practical challenge and may not be achieved efficiently. Despite these shortcomings, the capabilities of SMC still have a great potential for machine learning applications, enabling training and evaluation on the underlying full dataset. There are several studies on machine learning via SMC \cite{juvekar2018gazelle,mohassel2018aby3,kumar2020cryptflow}. In some cases, partial knowledge disclosure may be considered acceptable and traded for efficiency. For example, a SMC framework~\cite{knott2021crypten} is proposed to perform an efficient private evaluation of modern machine-learning models under a semi-honest threat model.

Differential privacy \cite{dwork2008differential} and k-Anonymity \cite{sweeney2002k} are used in another line of work for multiparty learning. These methods try to add noise to the data or obscure certain sensitive attributes until the third party cannot distinguish the individual. The disadvantage is that there is still heavy data transmission, which does not apply to large-scale training. In addition to transmitting encrypted data, there are studies on the encrypted transmission of parameters and training gradients, such as federated learning~\cite{yang2019federated} and swarm learning~\cite{warnat2021swarm}.
Federated learning was first proposed by Google and has been developed rapidly since then, wherein dedicated parameter servers are responsible for aggregating and distributing local training gradients. Besides, swarm learning is a data privacy-preserving framework that utilizes blockchain technology to decentralize machine learning-based systems. However, these methods usually can only deal with homogeneous local models \cite{NIPS2010_0d0fd7c6,rajkumar2012differentially}.

\begin{figure}
  \centering
  \includegraphics[width=0.4\textwidth]{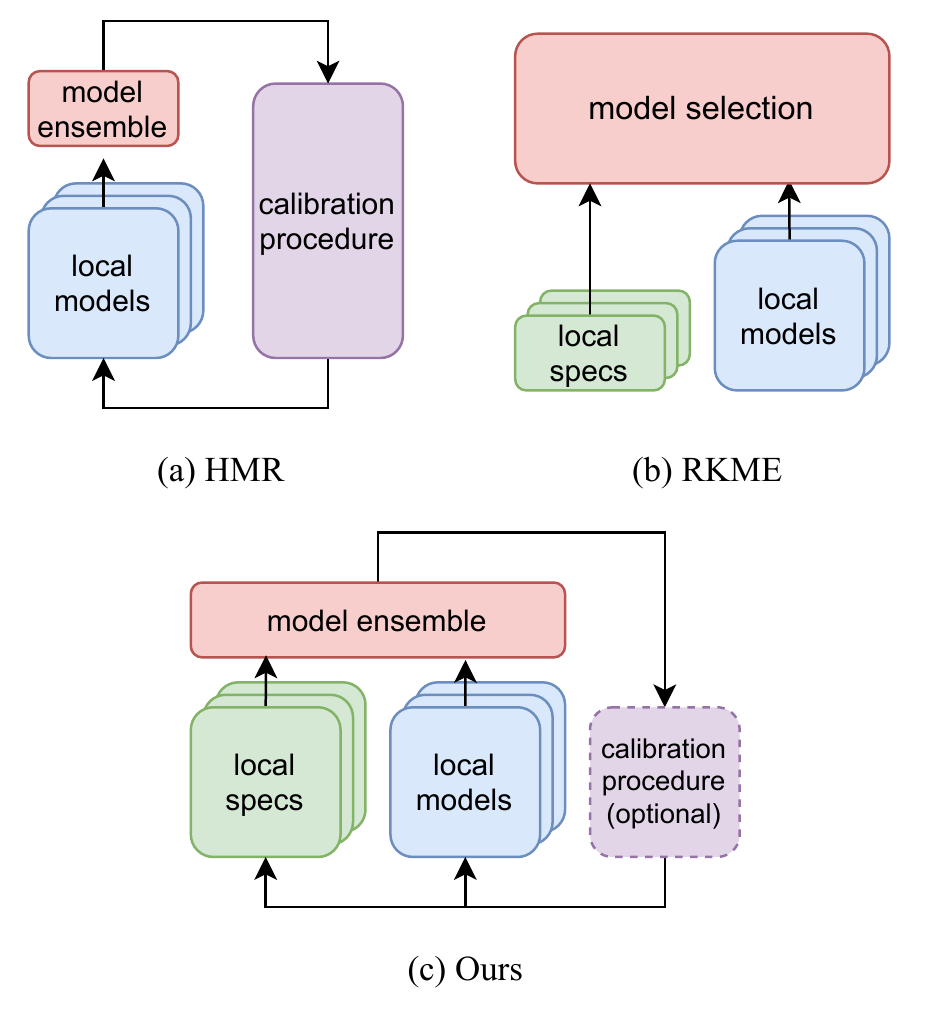}
  \caption[]{A comparison of our heterogeneous model reuse method with HMR~\cite{Wu2019} and RKME~\cite{9447164}. In HMR, multiple local models are simply combined and carefully calibrated to construct a global model. RKME does not require calibration, but some specifications that summarize local datasets are utilized for complicated model selection in the deployment phase. We utilize different types of specifications of local datasets in a different way, and design cheap aggregation strategy for model ensemble, where the calibration is optional due to satisfactory zero-shot test accuracy.}
  \label{fig:overview}
\end{figure}

\subsection{Model Reuse}
Model reuse aims to learn a reliable model for target task by reusing some related pre-trained models, often without accessing their original data~\cite{Zhou2016}.
Heterogeneous model reuse (HMR) for multiparty learning \cite{Wu2019} is the closest work to ours.
Based on the idea of learnware \cite{Zhou2016}, the black-box construction of a global model from the heterogeneous local models is performed. In HMR~\cite{Wu2019}, a global model is built based on the idea of Max-Model Predictor and then the model is carefully calibrated using a designed multiparty multiclass margin (MPMC-margin) loss. 
{
However, the accuracy is usually unsatisfactory under the zero-shot setting (model ensemble without any parameter or architecture calibration) due to the lack of exploitation of prior information.
In RKME~\cite{9447164}, each local classifier is assumed to be associated with a reduced kernel mean embedding as model specification, which largely improves the zero-shot test accuracy, but complicated model selection is required. Our method makes use of the data density specification (with privacy guarantee presented in section~\ref{sec:Privacy_Guarantee}), and a cheap model ensemble strategy is adopted to achieve very promising performance, even without calibration. Figure~\ref{fig:overview} is a comparison of our method with HMR and RKME.
}

The major difference between model reuse and some other related paradigms such as federated learning is that for the latter, information is exchanged among different parties in privacy-preserving ways during the training phase. While for model reuse, the training process of different parties is independent, and information is only exchanged and exploited in the form of models during the deployment phase \cite{yang2019federated,mcmahan2017communication,Ding2020}.

\section{The Proposed Method}

In this section, we first introduce the main notations and preliminaries of heterogeneous model reuse for the multiparty learning problem.

\subsection{Notations and Preliminaries}

We consider that there are $N$ participants in a multiparty learning system, and each participant $i\in [N]$ is known as a party and has its own local dataset $S_i = \{(x, y)\in\mathcal{X} \times\mathcal{Y}_i \}$ containing data samples and corresponding labels, where the labels are in $\mathcal{Y}_i\subseteq \mathcal{Y}$. Here, data exist in the form of isolated islands. Each party can only access its local dataset, so the underlying global dataset $S=\cup_{i=1}^N S_i$ cannot be directly observed by any parties. The participants attempt to cooperate in bridging the gap between model accuracy and training data accessibility, and obtaining a reliable global model. The whole model reuse progress is diagrammed as figure \ref{fig:overview}(d).

For a multiparty classification problem, each party $i$ holds its local classifier $\mathbf{F}_i: \mathcal{X} \to \mathcal{Y}_i$ which is trained on its local dataset $S_i$ and the types of classifiers can vary among parties.
The first challenge of learning the global model arises from the potential sample selection bias or covariate shift. A local classifier may misclassify an unseen input sample into the wrong class.
In fact, a local classifier would never be able to predict correctly if the local label space is not equal to the full label space, i.e. when $\mathbf{F}_i$ can only estimate posterior class probabilities $p(C_k|x, S_i)$ given $x$ for class $C_k\in\mathcal{Y}_i\subsetneq {\mathcal{Y}}$, we simply assign zero to $p(C_k|x, S_i)$ for $C_k \in \mathcal{Y} \setminus \mathcal{Y}_i$.

As for our method, in addition to the local classifier, each party should also fit a local density estimator $\mathbf{G}_i: \mathcal{X} \to \mathbb{R}$ on $S_i$ in an unsupervised manner.
The density estimator $\mathbf{G}_i$ is a generative probability model that learns to approximate the log-likelihood probability of the observations.
As we shall see in the following section, the log-likelihood and the dataset prior constitute the transition matrix that transforms the local class posterior probability to the global class posterior probability.
Therefore, the density estimators participate in the ensemble of local models in our model reuse framework together with the classifiers.
Besides, since $\mathbf{G}_i$ only provides the function of estimating the log-likelihood for given samples and does not need to generate samples, the privacy of the local dataset is guaranteed.

\subsection{Heterogeneous Model Reuse aided by Density Estimation}
We tackle the multiparty learning problem by utilizing some pre-trained local models to train a reliable global one.
Before combining local models, we shall dive into the decision theory of the multiparty learning problem to gain some insight.
We denote the joint probability distribution on the underlying global dataset as $p(x, C_k)$, and local joint distribution as conditional probability $p(x, C_k | S_i)$ given local dataset $S_i$.
The underlying global dataset is inaccessible and hence a directly estimation of $p(x, C_k)$ is intractable. A possible solution is to marginalize out $S_i$ to obtain $p(x,C_k)$:
\begin{equation}
  \label{eq:marginalization}
  p(x, C_k)
  = \mathbb{E}_{S_i\sim \mathcal{S}} \left[p(x, C_k | S_i)\right].
\end{equation}
For a classification task, we need to assign each observation $x$ to a certain class $C_k$. Such operation will divide the input space $\mathcal{X}$ into adjacent decision regions $\{\mathcal{R}_k\}$.
Our ultimate goal is to find the optimal decision policy $f^*\in\mathcal{X}\mapsto\mathcal{Y}$ that maximizes the probability of correct classification, i.e.,
\begin{equation}
  \label{eq:p_correct}
  P(\text{correct}) = \sum_{k} \int_{\mathcal{R}_k} p(x, f^*(x)) \, dx.
\end{equation}


It is straightforward to see that we can maximize the probability of correct classification if and only if we assign each $x$ to the class with the most considerable joint probability, since we can only assign $x$ to one class at a time.
Since the upper bound of Eq.~(\ref{eq:p_correct}) is $\int \max_{C_k} p(x,C_k) \,dx$~\cite{bishop2006pattern}, we have $f^*=\argmax_{C_k} p(\cdot,C_k)$.
By further expanding out $p(x, C_k)$ using marginalization Eq.~(\ref{eq:marginalization}), we can reformulate Eq.~(\ref{eq:p_correct}) as
\begin{equation}
  \label{eq:decision}
  P_{\max} = \int \sum_{i} p(C_k^* | x, S_i) p(x | S_i) p(S_i) \, dx,
\end{equation}
where $C_k^*=\argmax_{C_k} p(x,C_k)$.
In this way, we construct the global joint distribution by exploiting information about prior dataset distribution $p(S_i)$, local density/likelihood $p(x|S_i)$ and local class posterior $p(C_k | x, S_i)$.
To gain further insight into the global joint function, we multiply and divide the global likelihood $p(x)$ inner the right-hand integral, and rewrite Eq.~(\ref{eq:decision}) equivalently as

\vspace{-0.3cm}
\begin{align}
  \label{eq:global_posterior}
  & \int p(x) \underbrace{\left(\sum_{i} p(C_k^* | x, S_i) \frac{p(x | S_i) p(S_i)}{p(x)}\right)}_{p(C_k | x)} \, dx \\
  \label{eq:global_posterior_2}
           & =  \int p(x) \left(\sum_{i} p(C_k^* | x, S_i) \lambda_i \right) \, dx ,
\end{align}
where $\sum_{i=1}^N \lambda_i = 1$ and $\lambda_i = p(S_i|x)$ according to Bayes' theorem.
Compared with the original joint function Eq.~(\ref{eq:p_correct}), we now represent the global posterior probability $p(C_k| x)$ as a weighted sum of local posteriors.
Evidently, when there is only one party, $\lambda_1=1$, this joint representation degenerates to the familiar standard classification case, i.e., assigning $x$ to class $\argmax_{C_{k}} p(C_k|x)$.




When the dimension of input space $\mathcal{X}$ is small, estimation of the density $p(x|S_i)$ is trivial, and some popular and vanilla density estimation techniques, such as Gaussian mixture and kernel density estimators, from the classical unsupervised learning community can be directly adopted.
However, when the input dimension is high, such as in the case of image classification, landscape of the log-likelihood function $\mathbf{G}_i(x)$ for approximating the density can be extremely sharp due to the sparse sample distribution and thus intractable in practice.
We address this issue by extracting a sizeable common factor inner the summation, and rewriting the integrand in Eq.~(\ref{eq:decision}) equivalently as

\vspace{-5pt}
\begin{align}
  \label{eq:highdim}
  p(x|S_j) \sum_{i} p(C_k^* | x, S_i) p(S_i) e^{\log p(x | S_i) - \log p(x|S_j)},
\end{align}
where $j = \argmax_j \log p(x|S_{j})$.
In this way, we normalize the likelihood estimation to a reasonable interval $[0,1]$ without loss of information.

We then ensemble the local models according to Eq.~(\ref{eq:highdim}), as illustrated in Figure~\ref{fig:overview}(c), where $p(S_i)$ is proportional to the size of local dataset and sum up to 1, so that $p(S_i)=|S_i|/\sum_j |S_j|$.
Moreover, the class posterior $p(C_k|x,S_i)$ and density $p(x|S_i)$ can be approximated by the discriminate model $\mathbf{F}_i$ and generative model $\mathbf{G}_i$, respectively.
Finally, the global model can make a final decision, dropping the common factor $p(x|S_j)$, and the final decision policy can be written in a compact form by matrices as:
\begin{equation}
  \label{eq:obj}
  \argmax_{C}
  \left<\mathbf{F}^{(C)}(\cdot),
  \|\bm{S}\|_1 \odot
  \exp{\left(\mathbf{G}(\cdot) - {\bar{\mathbf{G}}(\cdot)} \right)}
  \right>,
\end{equation}
where $\bar{\mathbf{G}} = \argmax_i \mathbf{G}_i(\cdot)$ and $\odot$ is the Hadamard product.
Hereafter, we denote this inner product as the decision objective function $J(C)$ for simplicity.
The main procedure of our model reuse algorithm is summarized in Algorithm~\ref{alg:reuse}.

The following claim shows that our global model can be considered as a more general version of the max-model predictor defined in \cite{Wu2019}.

\begin{claim}
Let $\lambda_i=\delta_{\argmax_i p(x|S_i)}^i$, Eq.(\ref{eq:global_posterior_2}) would degenerate to a max-model predictor.
\end{claim}
\begin{proof}
If we assign $\lambda_i$ to $\delta_{\argmax_i p(x|S_i)}^i$, then according to Eq.~(\ref{eq:global_posterior_2}), we have
\begin{equation}
\label{eq:maxmodel}
    P = \int p(x) p(C_k^*|x,{\argmax_{S_i} p(x|S_i)}) \,dx.
\end{equation}
Recall the definition of $C_k^*$ and by dropping the common factor $p(x)$, the decision policy Eq.~(\ref{eq:maxmodel}) can be characterized as $\argmax_{C_k} \max_{S_i} p(C_k|x,S_i)$.
\end{proof}
By sharing selected samples among parties, HMR~\cite{Wu2019} makes $p(\cdot|S_i)$ get closer to each other in the Hilbert space so that $p(C_k|x,{\argmax_{S_i} p(x|S_i)})\rightarrow p(C_k|x)$.

\begin{algorithm}[!t]
  \caption{Heterogeneous Model Reuse aided by Density Estimation (without Calibration).}
  \label{alg:reuse}
  \textbf{Input}:
  \begin{itemize}[leftmargin=2em]
    \item Local classifiers $\mathbf{F}_1, \mathbf{F}_2, \dots, \mathbf{F}_N$ \\
          \null\hfill {$\triangleright$ e.g. CART, SVM, MLP, CNN\hspace*{1.5em}}
    \item Local log-likelihood estimators $\mathbf{G}_1, \mathbf{G}_2, \dots, \mathbf{G}_N$ \\
          \null\hfill {$\triangleright$ e.g. Kernel Density, Real-NVP, VAE\hspace*{1.5em}}
    \item Sizes of local datasets $|S_1|, |S_2|, \dots, |S_N|$
    \item Query samples ${x_1,x_2,\dots, x_m}$
  \end{itemize}
  \textbf{Output}: Labels of classification
  \begin{algorithmic}[1] 
    \FOR{$j=1,2,\dots,N$}
    \STATE initialize dataset prior probability by normalization:\\
    \hspace*{1em} $p_j:= |S_j|/\sum_i |S_i|$
    \FOR{$i=1,2,\dots,m$}
    \STATE 	for each class $k$ calculate local posterior probability, fill zeros for missing entries: \\
    \hspace*{1em} $\mathcal{G}_{ijk} := \mathbf{F}_j^{(C_k)}(x_i) \,\, \text{or} \,\, 0$
    \STATE calculate local log-likelihood:\\
    \hspace*{1em} $\mathcal{F}_{ij} := \mathbf{G}_j(x_i)$
    \ENDFOR
    \ENDFOR
    \FOR{$i=1,2,\dots,m$}
    \STATE calculate objective function for each class $k$: \\
    \hspace*{1em} $J_{ik} := \sum_j p_j \mathcal{F}_{ijk} \exp(\mathcal{G}_{ijk}-\max_j{\mathcal{G}_{ijk}})$\\
    \STATE make decision for query sample $x_i$:\\
    \hspace*{1em} $C_i:=\argmax_{k} J_{ik}$
    \ENDFOR
    \RETURN $\textbf{C}=C_1,C_2,\dots,C_m$
  \end{algorithmic}
\end{algorithm}

\subsection{Multiparty Cross-Entropy Loss}

In this subsection, we design a novel \emph{multiparty cross-entropy loss} (MPCE loss), which enables us to calibrate the classifiers in the compositional deep global model in an end-to-end manner.
We use $\theta$ and $\mu$ to denote the sets of classifiers' and generative models' parameters respectively, and we aim to find optimal $\theta$ so that we approximate the actual class posterior function well.
A popular way to measure the distance between two probabilities is to compute the Kullback-Leibler (KL) divergence between them.
With a slight abuse of notation, we characterize the KL divergence between true class posterior and approximated class posterior as

\vspace{-0.3cm}
\begin{align}
  \text{KL} & ( p \| p_{\theta})
  = \sum_{C\in\mathcal{C}} p(C|x) \log{\frac{p(C|x)}{p_\theta(C| x)}} \\
  \label{eq:KL}
            & = \mathbb{E}_{C\sim p} \left[\log p(C|x)\right]
  - \mathbb{E}_{C\sim p}\left[\log {p_{\theta}(C | x)} \right].
\end{align}
The first term in Eq.~(\ref{eq:KL}) is fixed and associated with the dataset.
Besides, as for a classification task, $p(\cdot|x)$ is a Kronecker delta function of class $C$, so this expectation term is $0$.
We define the \emph{MPCE loss} as the second term in Eq.~(\ref{eq:KL}), that is

\vspace{-0.3cm}
\begin{align}
  \mathcal{L}_{\text{mpce}}(\hat{y}, y)
   & = - \mathbb{E}_{C\sim  p}\left[\log {p_{\theta}(C | x)} \right] \\
   & = -\sum_{k} \delta_{k}^{y}  \log{p_{\theta}(C_k | x)},
\end{align}
where $\delta$ is the Kronecker delta function, $\delta_k^y$ is $1$ if $k$ and $y$ are equal, and $0$ otherwise.
By utilizing the global posterior presented in Eq.~(\ref{eq:global_posterior}), we can further expand out the loss to get

\vspace{-0.3cm}
\begin{align}
  \mathcal{L}_{\text{mpce}}(\hat{y}, y) = - \log \left\{\sum_{i} p_\theta(C_y | x, S_i) \frac{p_\mu(x | S_i) p(S_i)}{p(x)}\right\}.
\end{align}

\begin{claim}
For single party case, the MPCE loss degenerates to the standard cross-entropy loss.
\end{claim}
\begin{proof}
Evidently, when there is only one party, we have $p(x|S_i)p(S_i)=p(x,S_i)=p(x)$.
\end{proof}

Next, We follow the same argument about the high dimensional situation and apply the normalization trick presented in Eq.~(\ref{eq:highdim}), to obtain
\begin{equation}
\label{eq:l_mpce}
  \begin{split}
    &\mathcal{L}_{\text{mpce}}(\hat{y},y) =
    \left\{ - {\bar{\mathbf{G}}_{\mu}(x)} +\log p(x)\right\} \\
    &-\log
    \left<\mathbf{F}_{\theta}^{(C_y)}(x),
    \|\bm{S}\|_1 \odot
    \exp{\left(\mathbf{G}_{\mu}(x) - {\bar{\mathbf{G}}_{\mu}(x)} \right)}
    \right>.
  \end{split}
\end{equation}
Notice that in Eq.~(\ref{eq:l_mpce}), the last term in the $\log$ operation is the same as that in Eq.~(\ref{eq:obj}), minimizing the negative-log MPCE loss will maximize the policy objective function.
{
At step $t$, we can update the model parameters $[\theta,\mu]^\top$ using $- \eta g_t$, where $\eta$ is the learning rate hyperparameter and $g_t$ is the gradient defined as
\begin{equation} 
  g_t=\left[\begin{array}{cc}
      \nabla_\theta \mathcal{L}_{\text{mpce}}(\hat{y},y) &
      \sum_{\forall_i y\in \mathcal{Y}_i} \nabla_\mu \mathcal{L}_{\text{gen}}^{(i)}(x)
    \end{array}	\right]^\top.
\end{equation}
Here, $\mathcal{L}_{\text{gen}}^{(i)}$ is some unsupervised loss (such as negative log-likelihood for normalizing flows) of the $i$-th density estimator.
This can be conducted in a differential privacy manner by clipping and adding noise to $g_t$
\begin{equation} 
    \overline{g}_t = g_t / \max\{1,\|g_t\|_2/C\} + \mathcal{N}(0,\sigma^2C^2I).
\end{equation}
}
By optimizing the MPCE loss, the gradient information is back-propagated along a path weighted by the density estimation.
The party that provides more significant density estimation for calibration would obtain larger gradients and faster parameter updates.

\subsection{Privacy Guarantee}
\label{sec:Privacy_Guarantee}
In this paper, the data density is utilized as model specification, and this may lead to privacy issue.
However, since we can conduct density estimation in a differential privacy manner, the privacy can be well protected.

In particular, differential privacy~\cite{dwork2006calibrating,dwork2011firm,dwork2014algorithmic}, which is defined as follows, has become a standard for privacy analysis.
\begin{definition}[\cite{dwork2006calibrating}]
A randomized algorithm $\mathcal{M}: \mathcal{D} \mapsto \mathcal{R}$ satisfies $(\epsilon, \delta)-$differential privacy (DP) if and only if for any two adjacent input datasets $D, D'$ that differ in a single entry and for any subset of outputs $\mathcal{S}\subset \mathcal{R}$ it holds that
\begin{equation}
    \operatorname{Pr}[\mathcal{M}(D)\in \mathcal{S}]\leq e^{\epsilon} \operatorname{Pr}[\mathcal{M}(D')\in \mathcal{S}]+\delta.
\end{equation}
\end{definition}
It has been demonstrated that density estimators can be trained in a differential privacy manner to approximate arbitrary, high-dimensional distributions based on the DP-SGD algorithm~\cite{abadi2016deep,waites2021differentially}.
Therefore, the proposed model aggregation strategy is guaranteed to be $(\epsilon,\delta)$-differentially private when local models are pre-trained in a differential privacy manner, where $\epsilon$ and $\delta$ are the training privacy budget and training privacy tolerance hyperparameters, respectively.

\section{Experiments}

In this section, we evaluate our proposed method using a two-dimensional toy dataset and a popular benchmark dataset.
The basic experimental setup for our 2D toy and benchmark experiments is similar to that adopted in~\cite{Wu2019}.
Experiments on the benchmark data demonstrate our model reuse algorithm and end-to-end calibration process on various biased data distribution scenarios.
The code is available at \url{https://github.com/tanganke/HMR}.

\subsection{Toy Experiment}

\iftrue
\begin{figure}
  \centering
  {
  \begin{subfigure}[b]{0.15\textwidth}
    \centering
    \includegraphics[width=\textwidth]{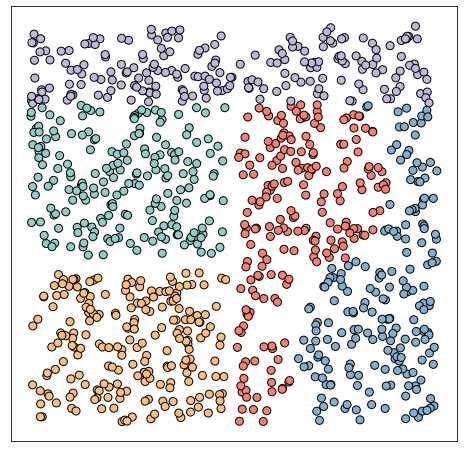}
    \caption{Train samples}
    \label{fig:toy_dataset}
  \end{subfigure}
  \begin{subfigure}[b]{0.15\textwidth}
    \centering
    \includegraphics[width=\textwidth]{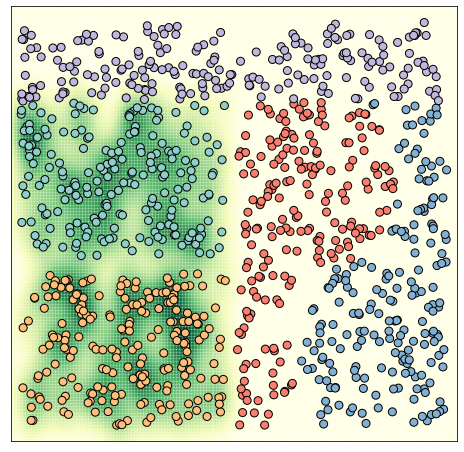}
    \caption{$p(x|S_1)$}
    \label{fig:toy_party_1_density}
  \end{subfigure}
  \begin{subfigure}[b]{0.15\textwidth}
    \centering
    \includegraphics[width=\textwidth]{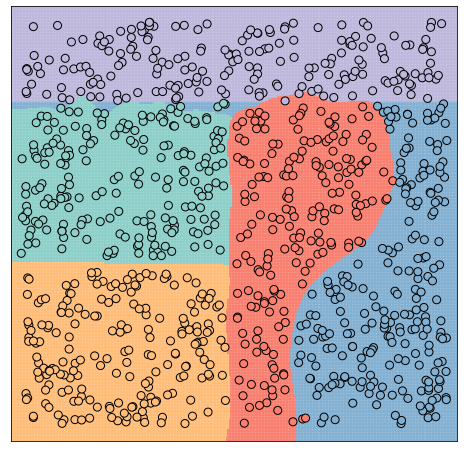}
    \caption{Ours-0: $99.1\%$}
    \label{fig:toy_decision_region}
  \end{subfigure}
  \\
  \begin{subfigure}[b]{0.15\textwidth}
        \centering
        \includegraphics[width=\textwidth]{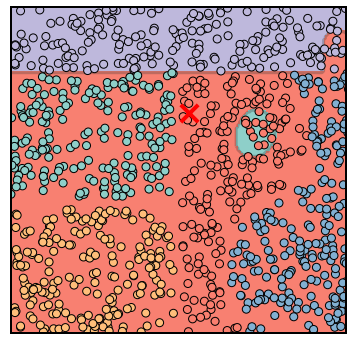}
        \caption{HMR-0: $42.4\%$}
        \label{fig:HMR_0}
  \end{subfigure}
  \begin{subfigure}[b]{0.15\textwidth}
        \centering
        \includegraphics[width=\textwidth]{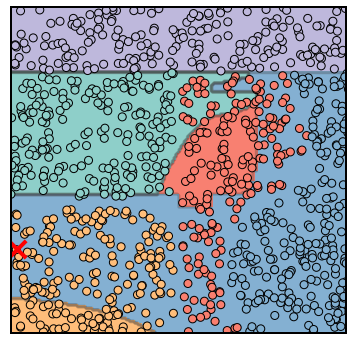}
        \caption{HMR-5: $71.2\%$}
        \label{fig:HMR_5}
  \end{subfigure}
  \begin{subfigure}[b]{0.15\textwidth}
        \centering
        \includegraphics[width=\textwidth]{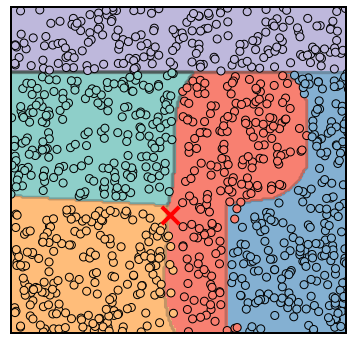}
        \caption{HMR-30: $98.9\%$}
        \label{fig:HMR_30}
  \end{subfigure}
  }
  \caption{
    Visualization results of 2D toy example.
    (a) The five-class 2D dataset.
    (b) The estimated density by party 1.
    (c) Decision boundary and accuracy of our method without calibration (iteration 0) on the testing data.
    (d-f) Decision boundary and test accuracy of HMR at iteration 0, 5 and 30.
    }
  \label{fig:toy_exp}
\end{figure}

\else

\begin{figure*}[!t]
  \centering
  \begin{subfigure}[b]{0.23\textwidth}
    \centering
    \includegraphics[width=\textwidth]{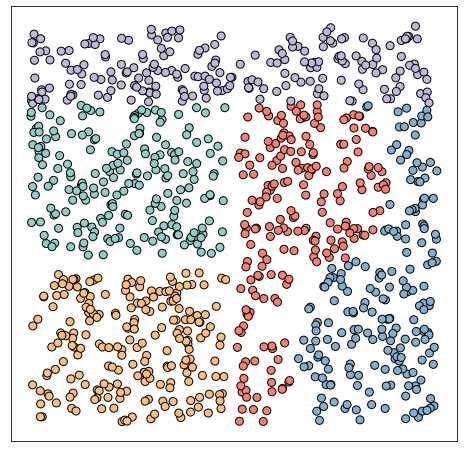}
    \caption{}
    \label{fig:toy_dataset}
  \end{subfigure}
  \begin{subfigure}[b]{0.23\textwidth}
    \centering
    \includegraphics[width=\textwidth]{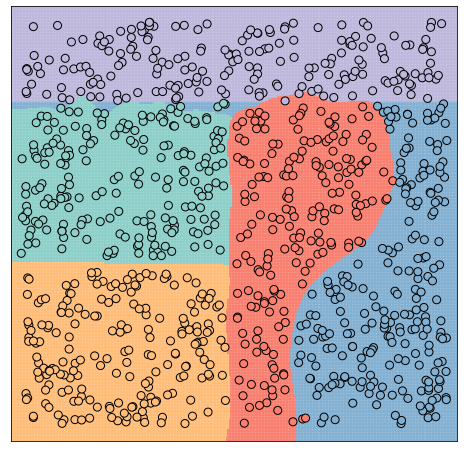}
    \caption{}
    \label{fig:toy_decision_region}
  \end{subfigure}
  \begin{subfigure}[b]{0.23\textwidth}
    \centering
    \includegraphics[width=\textwidth]{images/toy_obj_1.png}
    \caption{}
    \label{fig:toy_obj_begin}
  \end{subfigure}
  \begin{subfigure}[b]{0.23\textwidth}
    \centering
    \includegraphics[width=\textwidth]{images/toy_obj_2.png}
    \caption{}
  \end{subfigure}
  \begin{subfigure}[b]{0.23\textwidth}
    \centering
    \includegraphics[width=\textwidth]{images/toy_obj_3.png}
    \caption{}
  \end{subfigure}
  \begin{subfigure}[b]{0.23\textwidth}
    \centering
    \includegraphics[width=\textwidth]{images/toy_obj_4.png}
    \caption{}
  \end{subfigure}
  \begin{subfigure}[b]{0.23\textwidth}
    \centering
    \includegraphics[width=\textwidth]{images/toy_obj_5.png}
    \caption{}
    \label{fig:toy_obj_end}
  \end{subfigure}
  \begin{subfigure}[b]{0.23\textwidth}
    \centering
    \includegraphics[width=\textwidth]{images/toy_maxobj.png}
    \caption{}
    \label{fig:toy_obj_max}
  \end{subfigure}
  \caption{
    Visualization results of 2D toy example.
    (a) The five-class 2D dataset.
    (b) Decision boundary on testing dataset by the compositional global model.
    (c-g) The decision objective function $J(C)$ for each class $C\in\mathcal{Y}$. They share a common colorbar with (h).
    (h) The maximum objective function over classes, i.e. $\max_{C} J(C)$.
  }
  \label{fig:toy_exp}
\end{figure*}
\fi

\begin{figure}[!t]
  \centering
  \begin{subfigure}[b]{0.20\textwidth}
    \centering
    \includegraphics[width=\textwidth]{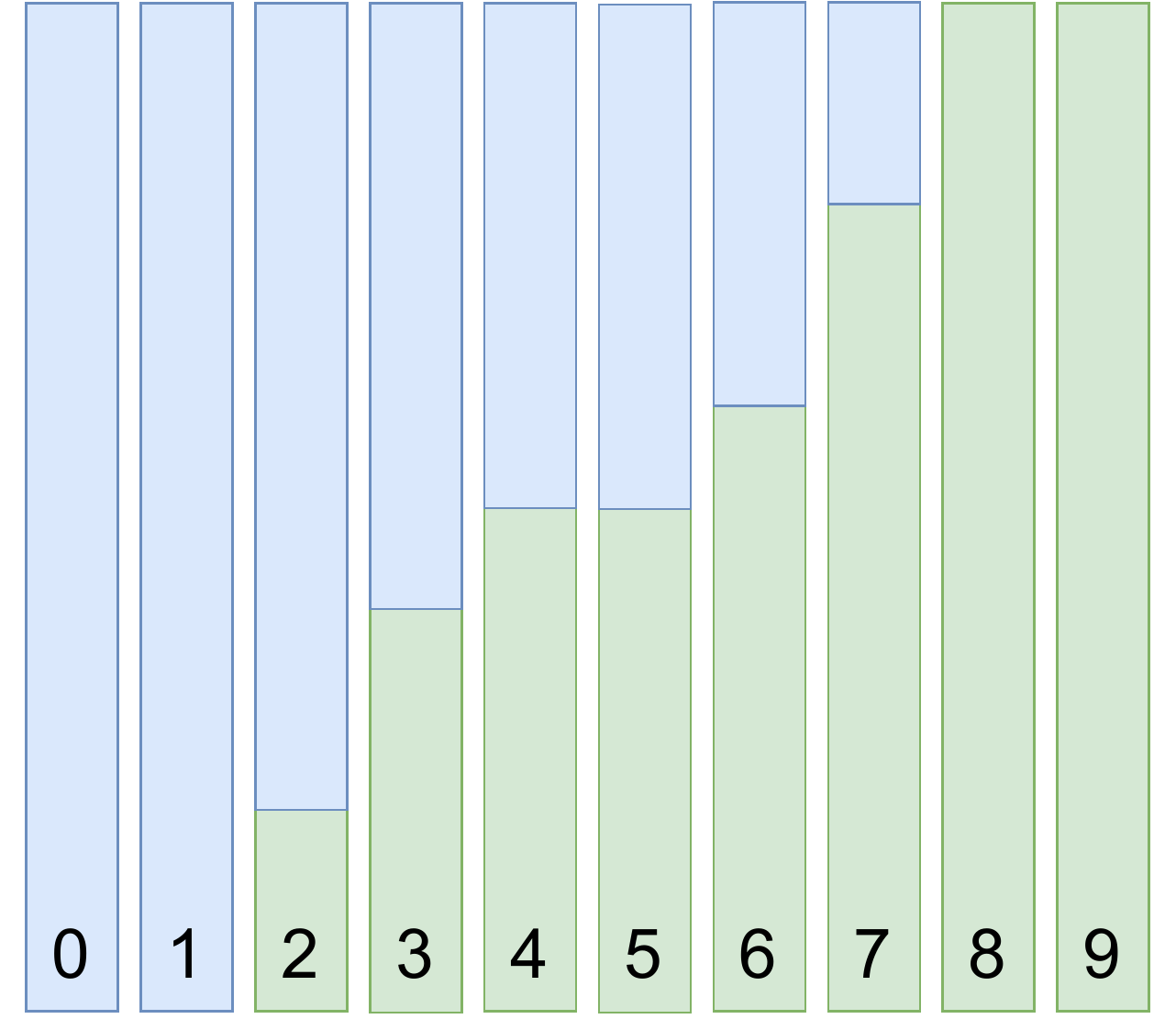} 
    \caption{A: 2 parties}
  \end{subfigure}
  \begin{subfigure}[b]{0.20\textwidth}
    \centering
    \includegraphics[width=\textwidth]{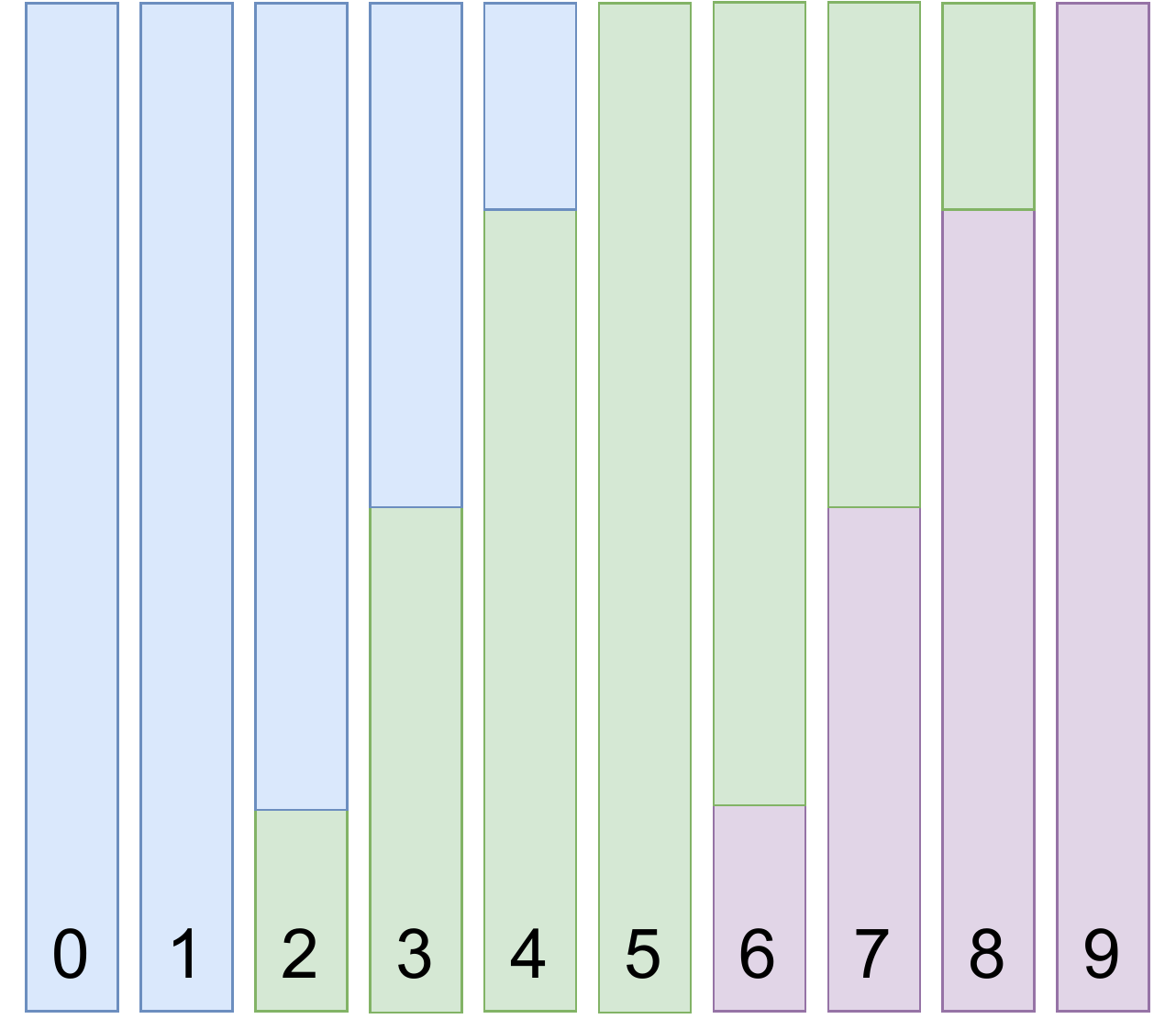} 
    \caption{B: 3 parties}
  \end{subfigure}
  \\
  \begin{subfigure}[b]{0.20\textwidth}
    \centering
    \includegraphics[width=\textwidth]{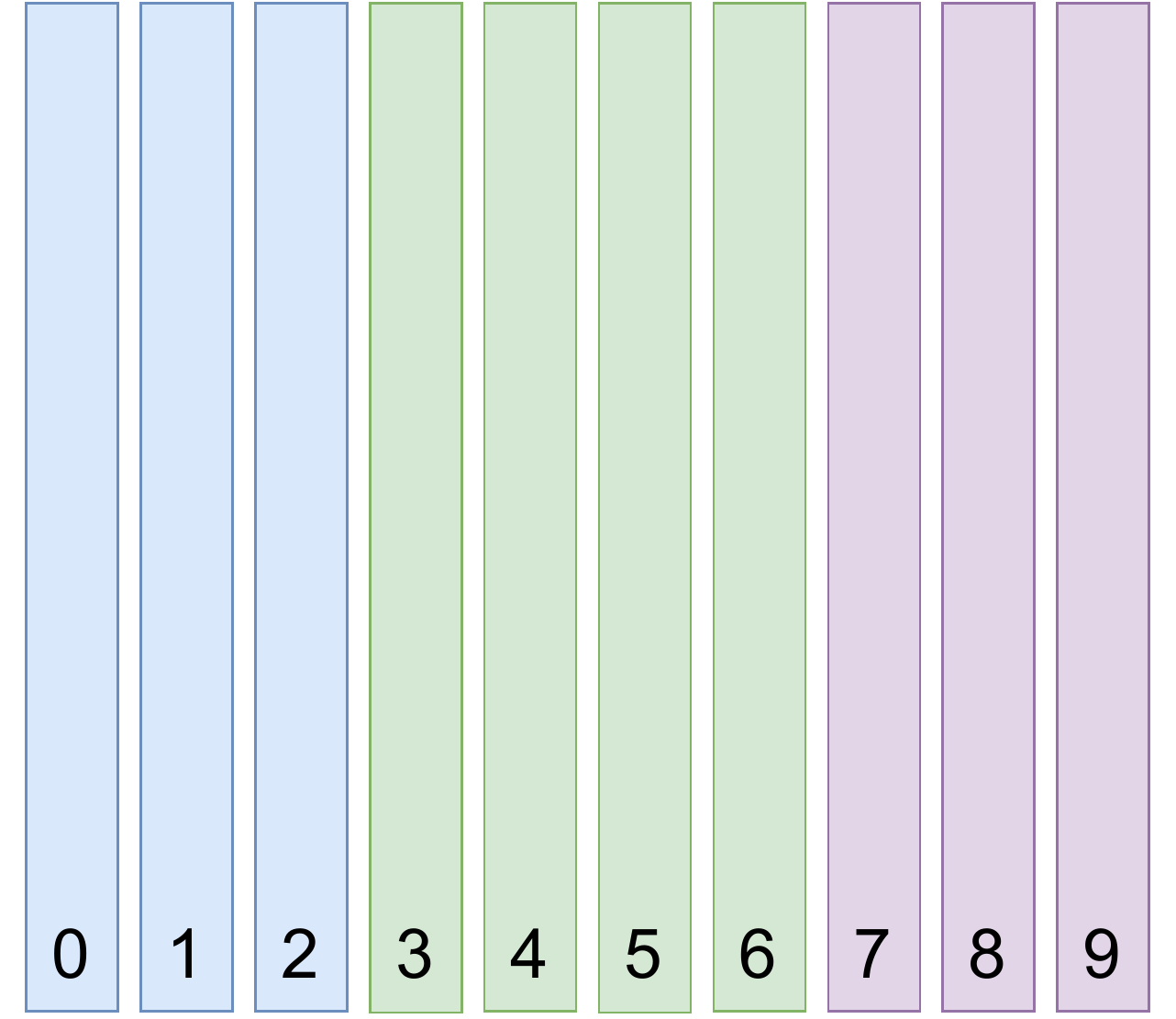} 
    \caption{C: 3 parties}
  \end{subfigure}
  \begin{subfigure}[b]{0.20\textwidth}
    \centering
    \includegraphics[width=\textwidth]{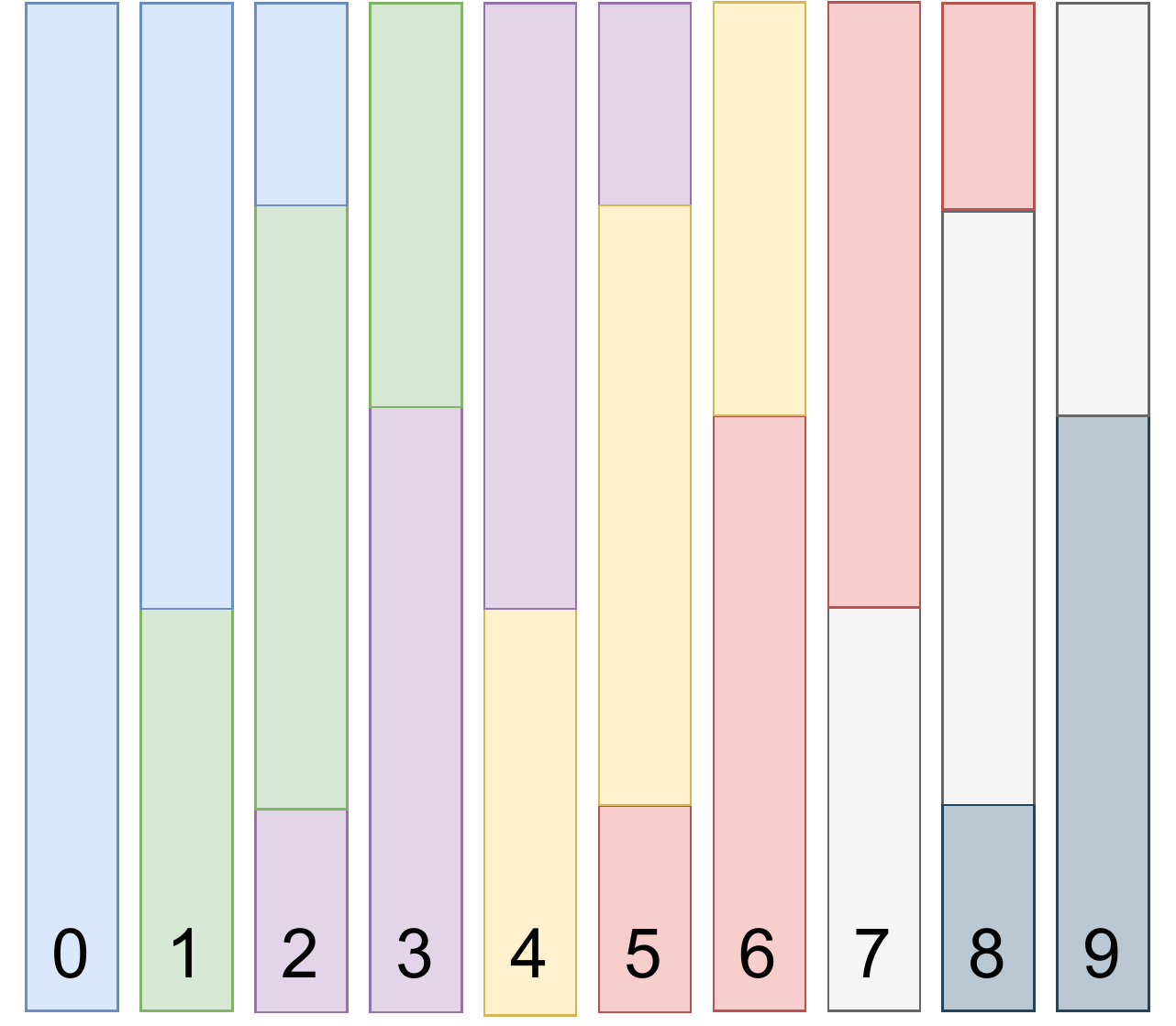} 
    \caption{D: 7 parties}
  \end{subfigure}
  \caption{
    Four experiment settings with different sample selection biases by dividing the training set of Fashion-MNIST. Each color represents a local dataset associated with a party.
  }
  \label{fig:benchmarking_setting}
\end{figure}

\begin{figure*}[!t]
  \centering{
    \includegraphics[width=0.85\textwidth]{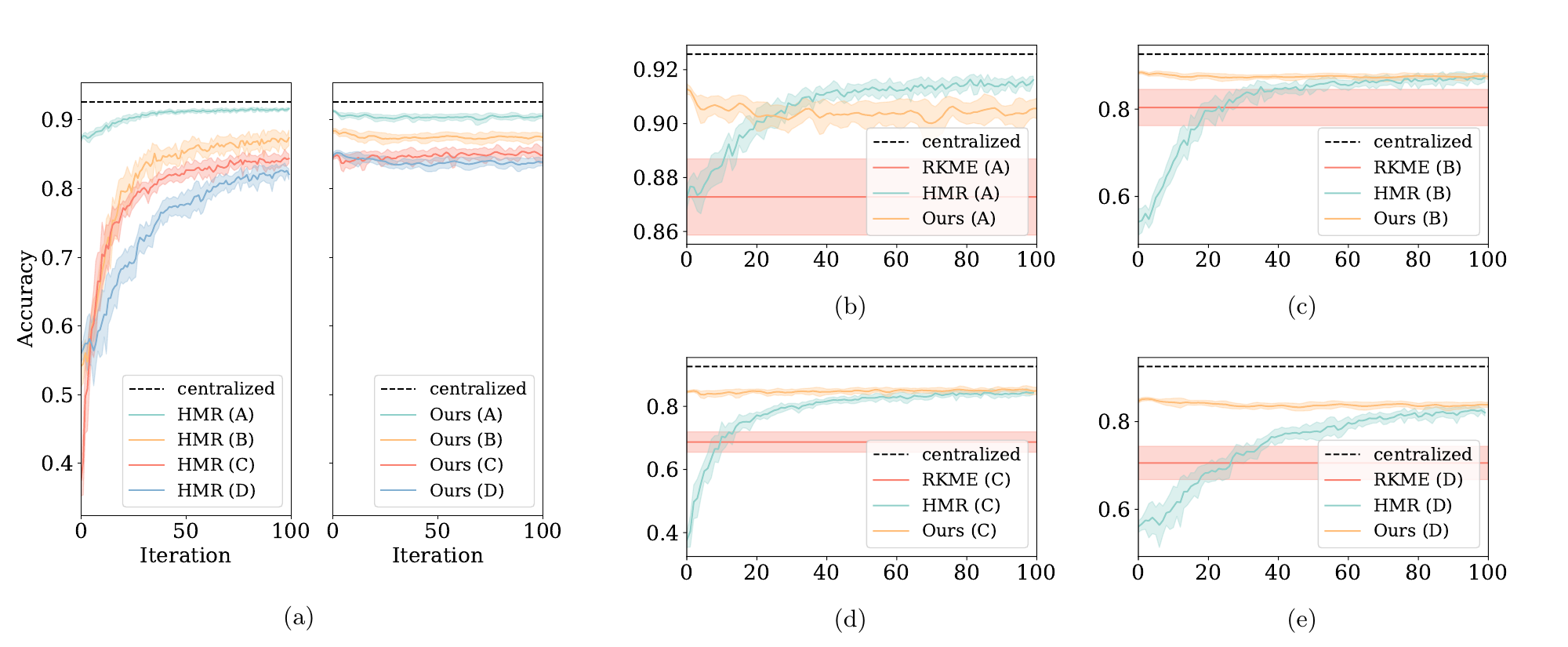}
  }
  \caption{The test accuracy curves over calibration iterations ($\text{avg.}\pm \text{std.}$).
    (a) overall performance of HMR (left) and our method (right) under the four multiparty settings.
    (b-e) performance of different compared approaches under each multiparty setting, where RKME is a constant value since it is inherently a method that cannot be subsequently calibrated.
  }
  \label{fig:benchmarking_result}
\end{figure*}

\begin{table*}[!t]
  \centering
  \begin{tabular}{c|cccc|c}
    \hline
    Setting & A                   & B                   & C                   & D                   & Average     \\
    \hline\hline
    \iftrue 
    RKME    & $87.3\pm1.4$      & $80.4\pm4.2$        & $68.7\pm3.2$        & $70.6\pm3.8$        & $76.7$      \\
    \fi
    HMR-1   & $87.3\pm0.4$      & $54.2\pm3.0$        & $37.6\pm2.3$        & $56.0\pm1.3$        & $58.8$      \\
    HMR-10  & $88.3\pm0.6$      & $66.0\pm3.1$        & $66.5\pm4.3$        & $59.4\pm4.0$        & $70.1$      \\
    HMR-50  & $91.2\pm0.3$      & $85.1\pm1.9$        & $82.1\pm0.9$        & $77.6\pm1.4$        & $84.0$      \\
    HMR-100 & $\bm{91.6\pm0.1}$ & $87.4\pm1.1$        & $84.4\pm0.6$        & $82.1\pm0.8$        & $86.4$      \\
    Ours    & $91.3\pm0.2$      & $\bm{88.4\pm0.4}$   & $\bm{84.6\pm0.3}$   & $\bm{84.7\pm0.4}$   & $\bm{87.2}$ \\
    \hline
  \end{tabular}
  \caption{
    Accuracy on benchmark data under four multiparty settings ($\text{avg.}\pm \text{std.}\%$).
    Here HMR-X represents the HMR method that has been calibrated for X rounds.
  }
  \label{tab:benchmarking_result}
\end{table*}

We first visualize our proposed method with a toy example.
Here, we create a 2D toy dataset with $2000$ points, each associated with a label from $5$ classes denoted by different colors.
The dataset is equally split into a training set and a test set, as shown in Figure~\ref{fig:toy_dataset}.

There are three parties in this toy example, each equipped with different local models.
The three parties use logistic regression, Gaussian kernel SVM, and gradient boosting decision tree for classification, respectively.
In addition, they all use kernel density estimator with bandwidth set to $0.1$ to estimate the log-likelihood function.
We implement the local models using the scikit-learn package \cite{scikit-learn}.
Each party can only access a subset of the complete training set as the local dataset.
The accessible samples are all the green and orange ones for party 1, all the red and blue ones for party 2, and all the blue and purple ones for party 3.


We first train the classifiers in a supervised manner and the kernel density estimators in an unsupervised manner on the corresponding local dataset.
Then we reuse these trained local models according to Algorithm~\ref{alg:reuse} to make final decisions on test samples.
Lastly, we analyze the zero-shot composition performance (without calibration) and compare with the most related work HMR~\cite{Wu2019}. The results are shown by Figure~\ref{fig:toy_exp}. From the results, we can see that the zero-shot composition accuracy reaches $99.1\%$, and the decision boundary is shown in figure~\ref{fig:toy_decision_region}. In contrast, the zero-shot accuracy of HMR is only $42.4\%$ and the performance is comparable to our method after $30$ rounds of calibrations.

\subsection{Benchmark Experiment}

In this set of experiments, we aim to understand how well our method compares to the state-of-the-art heterogeneous model reuse approaches for multiparty learning and the strength and weakness of our calibration procedure.
Specifically, we mainly compare our method with HMR~\cite{Wu2019} and RKME~\cite{9447164}.
\begin{itemize}
  \item HMR uses a max-model predictor as the global model together with a designed multiparty multiclass margin loss function for further calibration.
  \item RKME trains local classifiers and computes the \textit{reduced kernel mean embedding} (RKME) specification in the upload phase, assigns weights to local classifiers based on RKME specification, and trains a model selector for future tasks in the deployment phase.
\end{itemize}
In addition to multiparty learning, we train a centralized model on the entire training set for comparison.

We evaluate our method, 
HMR, and RKME on Fashion-MNIST \cite{xiao2017}, a popular benchmark dataset in the machine learning community, containing $70,000$ $28\times28$ gray-scale fashion product images, each associated with a label from $10$ classes.
The complete training set is split into a training set of $60,000$ examples and a test set of $10,000$ examples.
To simulate the multiparty setting, we separate the training set into different parties with biased sample distribution. The resulting four cases are shown as figure~\ref{fig:benchmarking_setting}, and we refer to ~\cite{Wu2019} for a detailed description.

\iftrue
We set the training batch size to be $128$, and the learning rate of all local models to $1$e-$4$ during the local training.
The learning rate is $1$e-$5$ during the calibration step.
All local classifiers have the same $3$-layer convolutional network  
and all local density estimators are the same $12$-layer real non-volume preserving (real NVP) flow network~\cite{Dinh2016}.
The real NVP network is a class of invertible functions and both the forward and its inverse computations are quite efficient. This enables exact and tractable density evaluation.
As For RKME, we set the reduced dimension size to $10$, and the number of generated samples to $200$.
\fi

Firstly, we test the zero-shot composition accuracy of the compared approaches, and if possible, evaluate the subsequent calibration performance. Due to the difference in the calibration mechanism, for HMR, a new neuron is added at the last layer of the classifiers to add reserved class output. In contrast, for our method, the calibration is end-to-end, and the structure of the classifiers is fixed. Therefore our method is more simple to implement. HMR retrains each local model on the augmented data set for one epoch during calibration. As for our method, the calibration operation is performed directly on the global model. Only a batch of $64$ data samples is randomly selected from the training set to perform gradient back-propagation. We run $20$ times for each setting to mitigate randomness and display the standard deviation bands. Experimental results including the centralized ones are visualized in Figure~\ref{fig:benchmarking_result}, and reported in Table~\ref{tab:benchmarking_result}.

From Figure~\ref{fig:benchmarking_result} and Table~\ref{tab:benchmarking_result}, we can see that for sufficient trained local models, our model reuse method achieves relatively superior accuracy from the beginning and outperforms all other model reuse counterparts. At the same time, subsequent calibrations do not improve performance or, even worse, slightly degrade performance. This may be because that the local models are well trained, the further calibration may lead to slight over-fitting. This demonstrates the effectiveness of our method that exploring data density for reuse.

\begin{figure}[!t]
  \centering{
    \begin{subfigure}{0.21\textwidth}
      \centering
      \includegraphics[width=\textwidth]{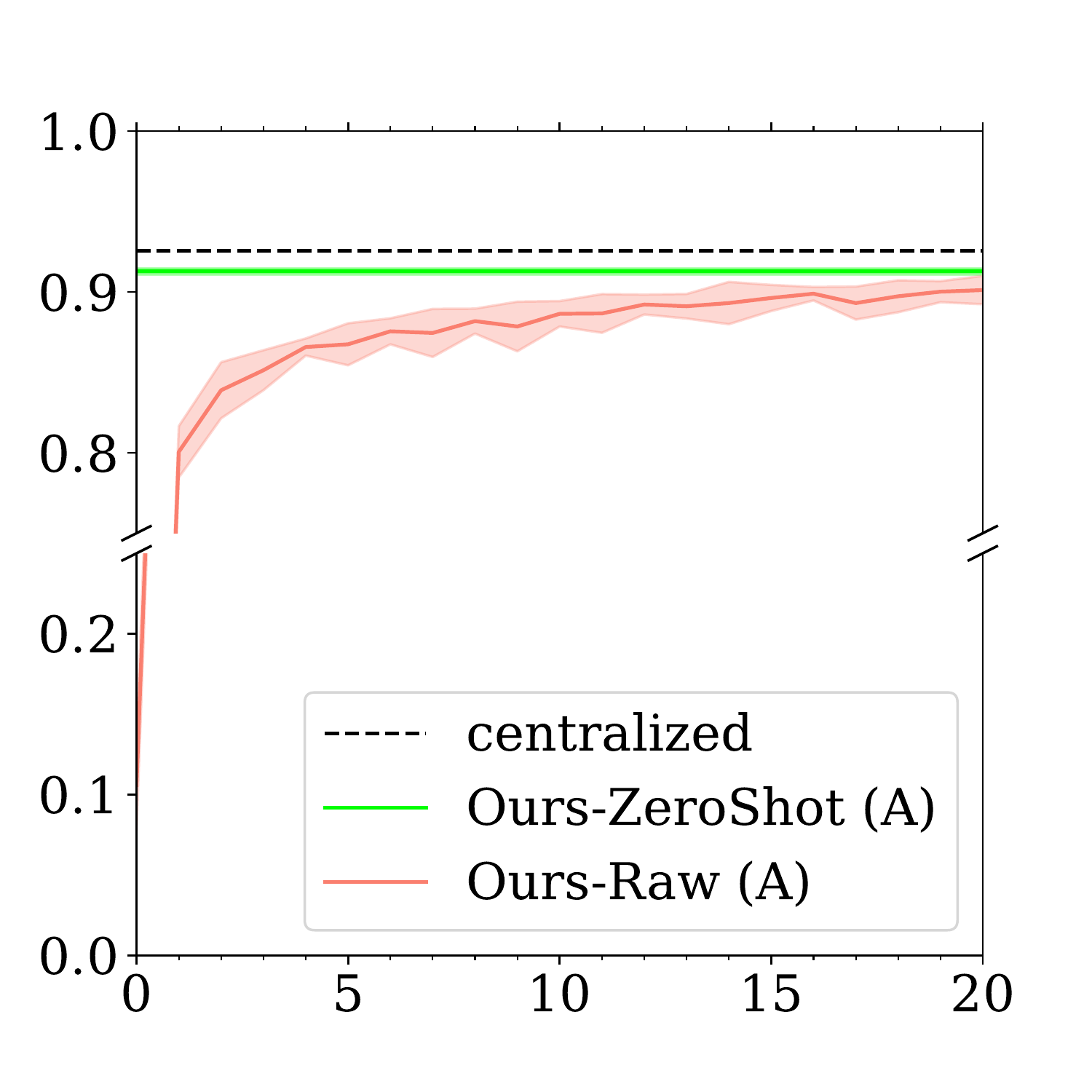}
      \caption{}
    \end{subfigure}
    \begin{subfigure}{0.21\textwidth}
      \centering
      \includegraphics[width=\textwidth]{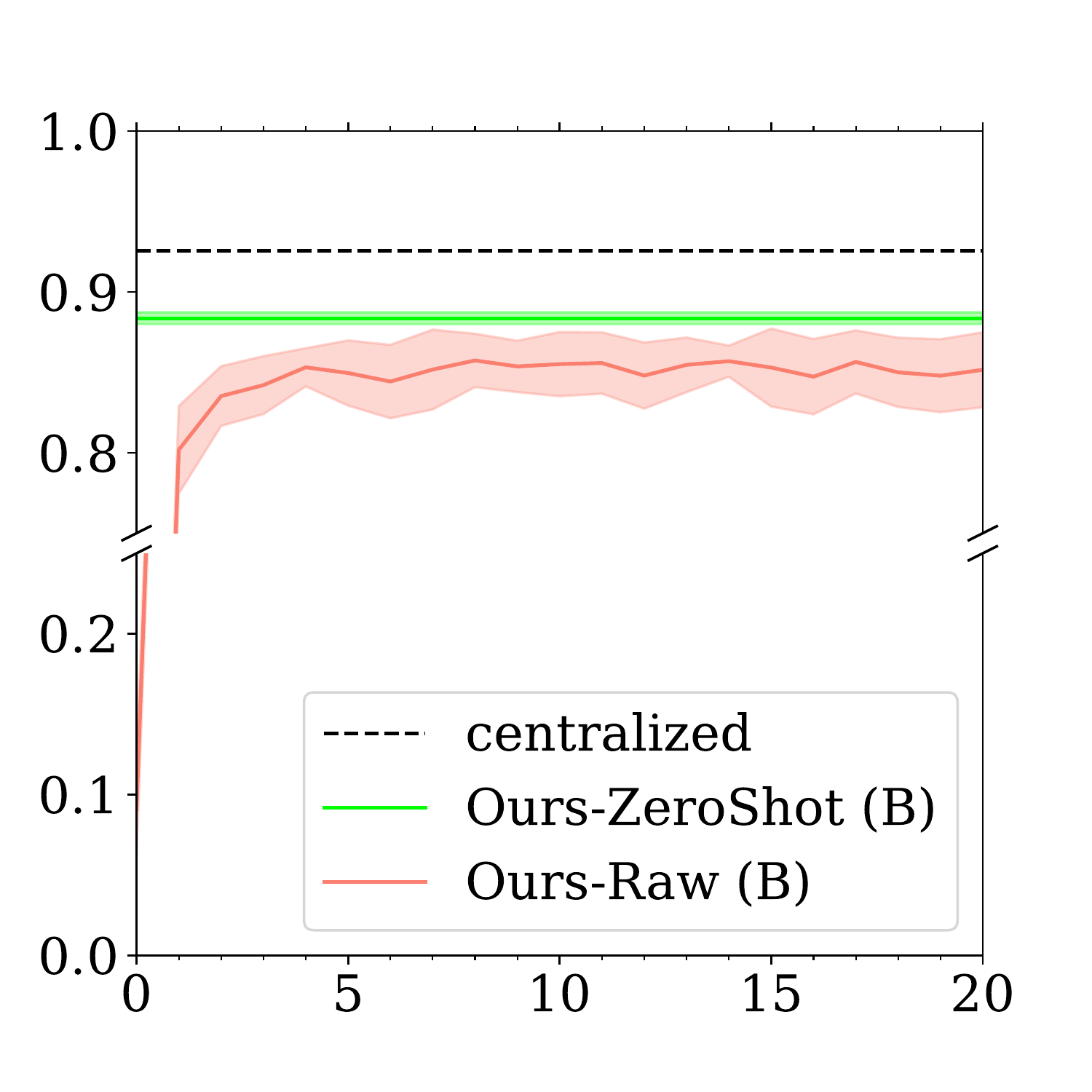}
      \caption{}
    \end{subfigure}
    \\
    \begin{subfigure}{0.21\textwidth}
      \centering
      \includegraphics[width=\textwidth]{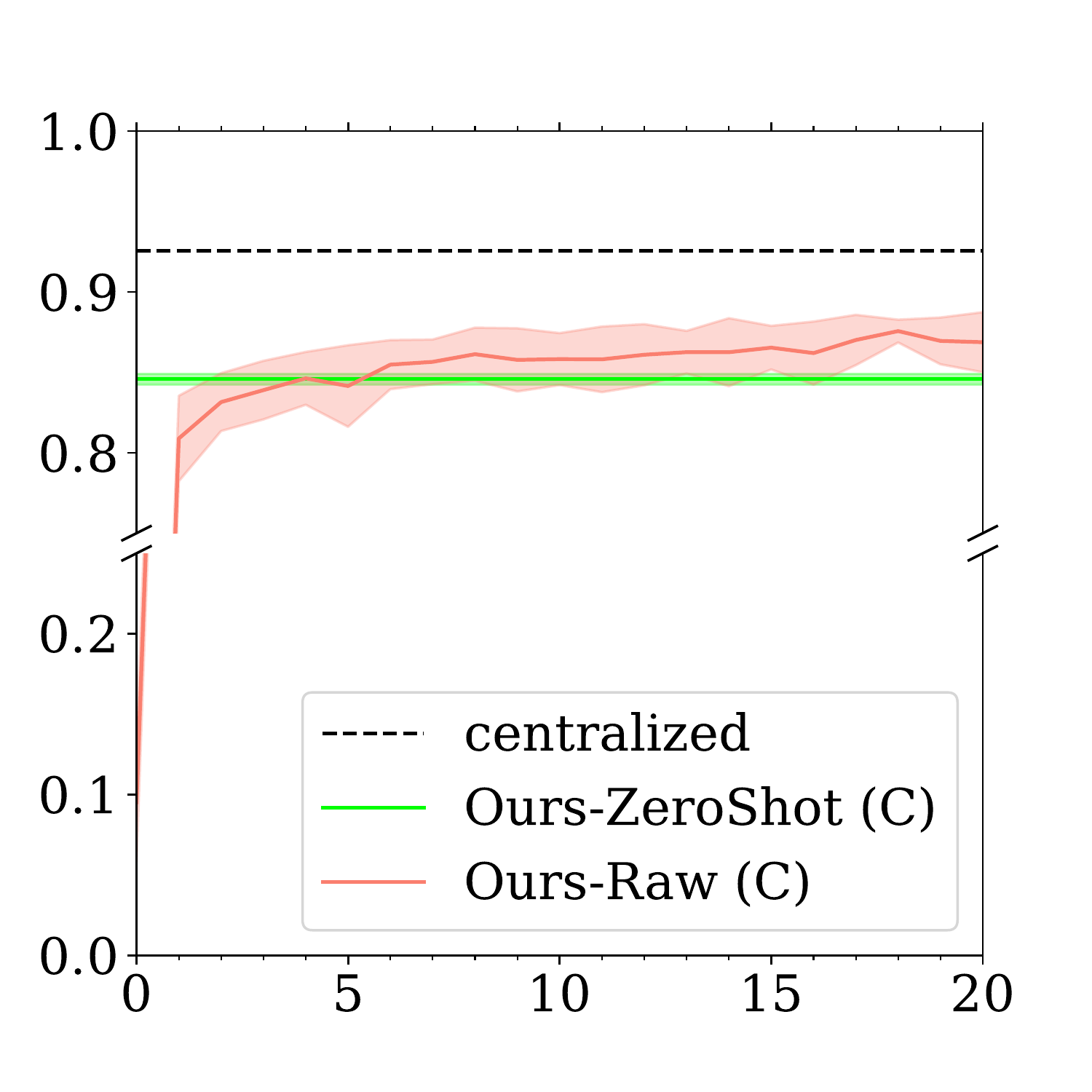}
      \caption{}
    \end{subfigure}
    \begin{subfigure}{0.21\textwidth}
      \centering
      \includegraphics[width=\textwidth]{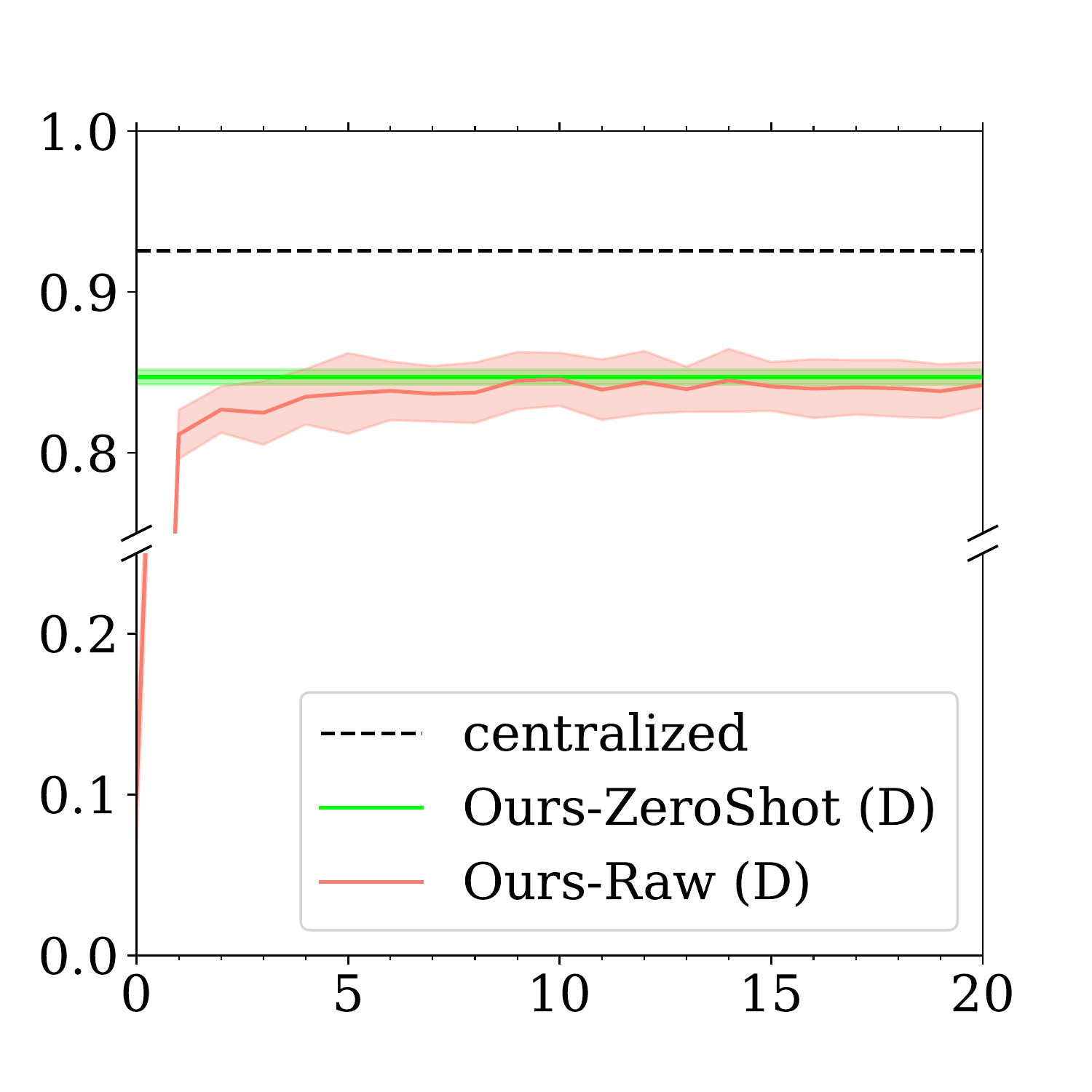}
      \caption{}
    \end{subfigure}
  }
  \caption{A comparison of the Raw models and the ZeroShot models ($\text{avg.}\pm \text{std.}$). Here, the Raw models represent the global models directly calibrated from random initialization without training, and the ZeroShot models represent the zero-shot composition from sufficient trained local models.
  }
  \label{fig:benchmarking_raw}
\end{figure}

Then we demonstrate that our calibration procedure is indeed effective when the local models are not well trained. In particular, we compare the test accuracy of the above zero-shot composition with the global model directly calibrated from random initialization without training (denoted as Raw). We fit the raw global models on the full training set for $20$ epochs with the learning rate set to be $1$e-$4$ and runs $20$ times for each multiparty setting. The results are shown as figure \ref{fig:benchmarking_raw}. We can observe from the results that during our calibration, the Raw model consistently achieves higher performance and eventually converges to the zero-shot composition accuracy.

\section{Conclusions}

In this paper, we propose a novel heterogeneous model reuse method for multiparty learning, where an auxiliary density estimator is designed to help the reuse.
In practical deployment, the pre-trained locals model can be provided as web query services, which is secure and privacy-friendly.
Besides, we propose a multiparty cross-entropy criteria to measure the distance between the true global posterior and the approximation.
Experimental results on both synthetic and benchmark data demonstrate the superiority of our method.
From the results, we mainly conclude that: 1) exploring more prior knowledge on the private local data during the training phase can lead to higher performance during the deployment phase; 2) substantial performance boost can be obtained by using the designed simple and easy-to-implement calibration strategy.
To the best of our knowledge, this is the first work to directly consider the multiparty learning problem from a decision theory perspective.

In the future, we plan to investigate the feature space to characterize and manipulate the knowledge learned from specific data. In addition to the popular image classification task, the proposed method can also be applied to tasks in other fields such as machine translation and speech recognition.

\clearpage
\section*{Acknowledgments}
This work was supported by the National Key Research and Development Program of China under Grant 2021YFC3300200, the Special Fund of Hubei Luojia Laboratory under Grant 220100014, the National Natural Science Foundation of China (Grant No. 62276195 and 62272354).
Prof Dacheng Tao is partially supported by Australian Research Council Project FL-170100117.

\bibliographystyle{named}
\bibliography{ijcai23}

\end{document}